\documentclass{article}

\usepackage{microtype}
\usepackage{booktabs} 
\usepackage{hyperref}       
\usepackage{url}            
\usepackage{amsfonts}   
\usepackage{dsfont}
\usepackage{fontenc}
\usepackage{amsmath}
\usepackage{amsthm}
\usepackage{amssymb}
\usepackage[toc, page]{appendix}
\usepackage{array}
\usepackage{authblk}
\usepackage{booktabs}
\usepackage{comment}
\usepackage{bbm}
\usepackage{fullpage}
\usepackage[T1]{fontenc}
\usepackage{mathtools}
\usepackage[utf8]{inputenc}
\usepackage{nicefrac}  
\usepackage[numbers]{natbib}
\usepackage{microtype}      
\usepackage{graphicx}
\usepackage{tabularx, booktabs, makecell, caption}
\usepackage[colorinlistoftodos]{todonotes}
\usepackage{siunitx}
\usepackage{multirow}
\usepackage{color}
\usepackage{wrapfig}
\usepackage{xcolor}
\usepackage{algorithm} 
\usepackage{algpseudocode} 
\usepackage{enumerate}
\usepackage{adjustbox}

\newtheorem{thm}{Theorem}

\newtheorem{lemma}{Lemma}
\newtheorem{cor}{Corollary}

\newtheorem{definition}{Definition}

\newcommand{\myparagraph}[1]{\textbf{#1}}

\usepackage{hyperref}

\usepackage{authblk}

\title{A Manifold View of Adversarial Risk}

\author[1]{Wenjia Zhang}
\author[2]{Yikai Zhang}
\author[3]{Xiaoling Hu}
\author[4]{Mayank Goswami}
\author[5]{Chao Chen}
\author[1]{Dimitris Metaxas}

\affil[1]{Department of Computer Science, Rutgers University} 
\affil[2]{Morgan Stanley} 
\affil[3]{Department of Computer Science, Stony Brook University} 
\affil[4]{Department of Computer Science, Queens College of CUNY}
\affil[5]{Department of Biomedical Informatics, Stony Brook University}

\begin{document}
\maketitle

\begin{abstract}
  The adversarial risk of a machine learning model has been widely studied. Most previous works assume that the data lies in the whole ambient space. We propose to take a new angle and take the manifold assumption into consideration. Assuming data lies in a manifold, we investigate two new types of adversarial risk, the normal adversarial risk due to perturbation along normal direction, and the in-manifold adversarial risk due to perturbation within the manifold. We prove that the classic adversarial risk can be bounded from both sides using the normal and in-manifold adversarial risks. We also show with a surprisingly pessimistic case that the standard adversarial risk can be nonzero even when both normal and in-manifold risks are zero. We finalize the paper with empirical studies supporting our theoretical results. Our results suggest the possibility of improving the robustness of a classifier by only focusing on the normal adversarial risk. 
\end{abstract}

\section{Introduction}

Machine learning (ML) algorithms have achieved astounding success in multiple domains such as computer vision~\citep{krizhevsky2012imagenet,he2016deep}, natural language processing~\citep{wu2016google,vaswani2017attention}, and robotics~\citep{levine2014learning,nagabandi2018neural}. These models perform well on massive datasets but are also vulnerable to small perturbations on the input examples. Adding a slight and visually unrecognizable perturbation to an input image can completely change the model's prediction. Many works have been published focusing on such adversarial attacks~\citep{szegedy2013intriguing,carlini2017towards,madry2017towards}. To improve the robustness of these models, various defense methods have been proposed~\citep{madry2017towards,zhang2019theoretically,shafahi2019adversarial}. These methods mostly focus on minimizing the \emph{adversarial risk}, i.e., the risk of a classifier when an adversary is allowed to perturb any data with an oracle. 

Despite the progress in improving the robustness of models, it has been observed that compared with a standard classifier, a robust classifier often has a lower accuracy on the original data. The accuracy of a model can be compromised when one optimizes its adversarial risk. This phenomenon is called \emph{the trade-off between robustness and accuracy}. \cite{su2018robustness} observed this trade-off effect on a large number of commonly used model architectures. They concluded that there is a linear negative correlation between the logarithm of accuracy and adversarial risk. \cite{tsipras2018robustness} proved that adversarial risk is inevitable for any classifier with a non-zero error rate. \cite{zhang2019theoretically} decomposed the adversarial risk into the summation of standard error and boundary error. The decomposition provides the opportunity to explicitly control the trade-off. They also proposed a regularizer to balance the trade-off by maximizing the boundary margin.

In this paper, we investigate the adversarial risk and the robustness-accuracy trade-off through a new angle. We follow the classic manifold assumption, i.e., data are living in a low dimensional manifold embedded in the input space \citep{rifai2011manifold,cayton2005algorithms,narayanan2010sample,niyogi2008finding}. 

Based on this assumption, we analyze the adversarial risk with regard to adversarial perturbations within the manifold and normal to the manifold. By restricting to in-manifold and normal perturbations, we define the \emph{in-manifold adversarial risk} and \emph{normal adversarial risk}. Using these new risks, together with the standard risk, we prove an upper bound and a lower bound for the adversarial risk. We also show that the bound is tight by constructing a pessimistic case. We validate our theoretical results using synthetic experiments.

Our study sheds light on a new aspect of the robustness-accuracy trade-off. Through the decomposition into in-manifold and normal adversarial risks, we might find an extra margin to exploit without confronting the trade-off. 
Future work will include developing normal adversarial training algorithms for real-world datasets.

\subsection{Related Works}
\myparagraph{Robustness-accuracy Trade-off}
There are several works studying the trade-off between robustness and accuracy~\citep{tsipras2018robustness,su2018robustness,zhang2019theoretically,dohmatob2019generalized}. The basic question is whether the trade-off actually exists. i.e. is there a classifier that is both accurate and robust? Empirical and theoretical proofs showed that actual trade-off does exist even in the infinite data limit~\citep{tsipras2018robustness,su2018robustness,zhang2019theoretically}. \cite{dohmatob2019generalized} showed that a high accuracy model can inevitably be fooled by the adversarial attack. \cite{zhang2019theoretically} gave examples showing that the Bayes optimal classifier may not be robust. 

However, some works have different views on this trade-off or even its existence. In contrast to the idea that the trade-off is unavoidable, these works argued that a lack of sufficient optimization methods~\citep{awasthi2019robustness,rice2020overfitting,shaham2018understanding} or better network architecture~\citep{guo2020meets,fawzi2018adversarial} causes the drop in accuracy, instead of the increase in robustness.  \cite{yang2020closer} showed the existence of both robust and accurate classifiers and argued that the trade-off is influenced by the training algorithm to optimize the model. They investigated distributionally separated dataset and claimed that the gap between robustness and accuracy arises from the lack of a training method that imposes local Lipschitzness on the classifier. Remarkably, in~\citep{gowal2020uncovering,raghunathan2020understanding,carmon2019unlabeled}, it was shown that with certain augmentation of the dataset, one may be able to obtain a model that is both accurate and robust.


\myparagraph{Manifold Assumption}
One important line of research focuses on the manifold assumption on the data distribution. This assumption suggests that observed data is distributed on a low dimensional manifold~\citep{rifai2011manifold,cayton2005algorithms,narayanan2010sample} and there exists a mapping that embeds the low dimension manifold in some higher dimension space. Traditional manifold learning methods~\citep{tenenbaum2000global,saul2003think} try to recover the embedding by assuming the mapping preserves certain properties like distances or local angles. Following this assumption, on the topic of robustness, \cite{tanay2016boundary} showed the existence of adversarial attack on the flat manifold with linear classification boundary. It was proved later in ~\cite{gilmer2018adversarial} that in-manifold adversarial examples exist. They stated that high dimension data is highly sensitive to $l_2$ perturbations and pointed out the nature of adversarial is the issue with potential decision boundary. Later, \cite{stutz2019disentangling} showed that with the manifold assumption, regular robustness is correlated with in-manifold adversarial examples, and therefore, accuracy and robustness may not be contradictory goals. Further discussion~\citep{xie2020adversarial} even suggested that adding adversarial examples in the training process can improve the accuracy of the model.
\cite{lin2020dual} used perturbation within a latent space to approximate in-manifold perturbation. To the best of our knowledge, no existing work discussed normal perturbation and normal adversarial risk as we do. We are also unaware of any theoretical results proving upper/lower bounds for adversarial risk in the manifold setting.

We also note a classic manifold reconstruction problem, i.e., reconstructing a $d$-dimensional manifold given a set of points sampled from the manifold. A large group of classical algorithms~\citep{edelsbrunner1994triangulating,dey2006provable,niyogi2008finding} are provably good, i.e., they give a guarantee of reproducing the manifold topology with a sufficiently large number of sample points.

\section{MANIFOLD BASED RISK DECOMPOSITION}

In this section, we state our main theoretical result~\ref{riskdecom}, which decomposes the adversarial risk into appropriately defined normal and in-manifold or tangential risks. We first define these quantities and set up basic notation, with the main theorem following in Section~2.3. For the sake of simplicity, we describe our main theorem in the setting of binary $\{-1,1\}$ labels.

\subsection{Data Manifold}\label{manifold}

Let $(\mathbb{R}^D, \vert\vert.\vert\vert)$ denote the $D$ dimensional Euclidean space with $\ell_{2}$-norm. For $x \in \mathbb{R}^{D}$, $B_{\epsilon}(x)$ be the open ball of radius $r$ in $\mathbb{R}^D$ with center at $x$. For a set $A \subset \mathbb{R}^{D}$, define $B_{\epsilon}(A)=\{y: \exists x\in A, d(x,y) < \epsilon\}$.

Let $\mathcal{M} \subset \mathbb{R}^D$ be a $d$-dimensional compact smooth manifold embedded in $\mathbb{R}^D$. Thus for any $x \in M$ there is a corresponding coordinate chart $(U, g)$ where $U\ni x$ is a open set of $\mathcal{M}$ and $g$ is a homeomorphism from $U$ to a subset of $\mathbb{R}^d$. 
For $x\in \mathcal{M}$, we let $T_{x}\mathcal{M}$ and $N_{x}\mathcal{M}$ denote the tangent and normal spaces at $x$. Intuitively, the tangent space $T_{x}\mathcal{M}$ is the space of tangent directions, or equivalence classes of curves in $\mathcal{M}$ passing through $x$, with two curves considered equivalent if they are tangent at $x$. The normal space $N_{x}\mathcal{M}$ is the set of vectors in $\mathbb{R}^{D}$ that are orthogonal to any vector in $T_{x}\mathcal{M}$. Since $\mathcal{M}$ is a smooth $d$-manifold, $T_{x}\mathcal{M}$ and $N_{x}\mathcal{M}$ are $d$ and $D-d$ dimensional vector spaces, respectively. See Figure~\ref{fig:smooth_manifold}. For detailed definitions, we refer the reader to \cite{bredon2013topology}.

We assume that the data and (binary) label pairs are drawn from $\mathcal{M} \times \{-1,1\}$  according to some unknown distribution $p(x, y)$. Note that $\mathcal{M}$ is unknown. A score function $f(x)$ is a continuous function from $\mathbb{R}^{D}$ to $[0,1]$. We denote by $\mathds{1}(A)$ the indicator function of the event $A$ that is $1$ if $A$ occurs and $0$ if $A$ does not occur, and will use it to represent the 0-1 loss.

\begin{figure}
\vspace{-.2in}
    \centering
    \includegraphics[width =0.55\textwidth]{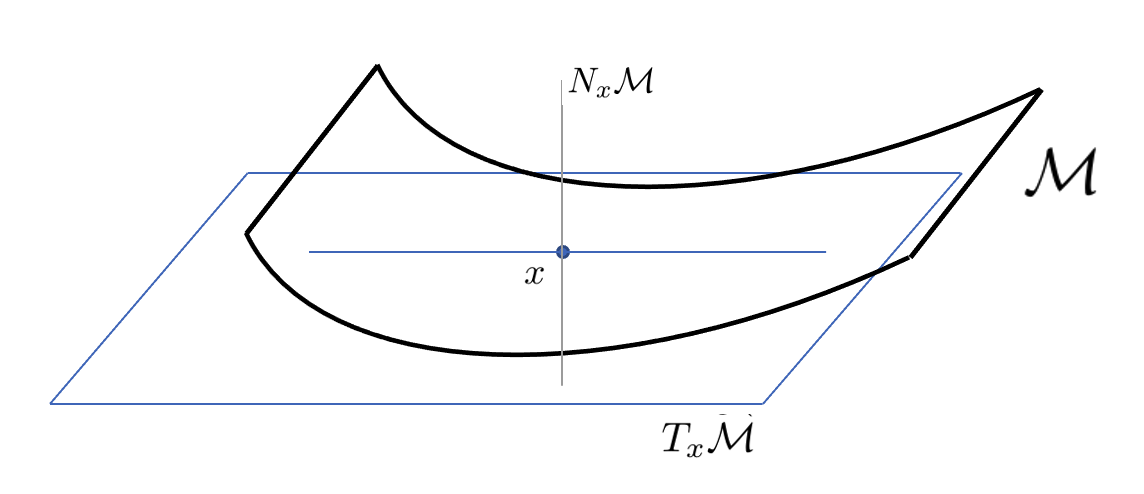}
\vspace{-.2in}
    \caption{Tangential and Normal Space}
\vspace{-.2in}
    \label{fig:smooth_manifold}
\end{figure}

\subsection{Robustness and Risk}\label{riskdef}

Given data from $\mathcal{M} \times \{-1,1\}$ drawn according to $p$ and a classifier $f$ on $\mathbb{R}^{D}$, we define three types of risks. The first, adversarial risk, has been extensively studied in machine learning literature:


\begin{definition}[Adversarial Risk]
Given $\epsilon>0$, define the adversarial risk of classifier $f$ with budget $\epsilon$ to be
\[
R_{adv}(f, \epsilon) := \displaystyle\mathop{\mathbb{E}}_{(x, y)\sim p}\mathds{1}(\exists x^{\prime}\in B_{\epsilon}(x): f(x')y \leq 0)
\]
\end{definition}

Notice that $B_{\epsilon}(x)$ is the open ball around $x$ in $\mathbb{R}^{D}$ (the ambient space).

We next define risk that is concerned only with in-manifold perturbations. Previously, \cite{gilmer2018adversarial} and \cite{stutz2019disentangling} showed that there exist in-manifold adversarial examples, and empirically demonstrated that in-manifold perturbations are a cause of the standard classification error. Therefore, in the following, we define the in-manifold perturbations and in-manifold adversarial risk.

\begin{definition}[In-manifold Risk]
Given $\epsilon>0$, the in-manifold adversarial perturbation for classifier $f$ with budget $\epsilon$ is the set
\[
B_{\epsilon}^{in}(x) := \{x^{\prime}\in \mathcal{M}: \|x-x^{\prime}\|\leq \epsilon\}
\]
The in-manifold adversarial risk is 
\[
R_{adv}^{in}(f, \epsilon) := \displaystyle\mathop{\mathbb{E}}_{(x, y)\sim p}\mathds{1}(\exists x^{\prime}\in B^{in}_{\epsilon}(x): f(x')y \leq 0)
\]
\end{definition}

We remark that while the above perturbation is on the manifold, in many manifold-based defense algorithms use generative models to estimate the homeomorphism (the manifold chart) $z=g(x)$ for real-world data.  Therefore, instead of in-manifold perturbation, one can also use an equivalent $\eta$-budget perturbation in the \textit{latent} space. However, for our purposes, the in-manifold definition will be more convenient to use. Lastly, we define the \textit{normal} risk:

\begin{definition}[Normal Adversarial Risk]
Given $\epsilon>0$, the normal adversarial perturbation for classifier $f$ with budget $\epsilon$ is be the set
\[
B^{nor}_{\epsilon}(x):=\{x': x'-x\in N_{x}\mathcal{M},  | \|x-x^{\prime}\|\leq \epsilon\}
\] 
Define the normal adversarial risk as
\[
R_{adv}^{nor}(f, \epsilon) := \displaystyle\mathop{\mathbb{E}}_{(x, y)\sim p}\mathds{1}(\exists x^{\prime}\neq x \in B^{nor}_{\epsilon}(x): f(x')y \leq 0)
\]
\end{definition}

Notice that the normal adversarial risk is non-zero if there is an adversarial perturbation $x' \neq x$ in the normal direction at $x$. Finally, we have the usual \emph{standard risk}:
$R_{std}(f) := \displaystyle\mathop{\mathbb{E}}\nolimits_{(x, y)\sim p}\mathds{1}( f(x)y \leq 0)$.


\subsection{Main  Result: Decomposition of Risk}

In this section, we state our main result that decomposes the adversarial risk into its tangential and normal components. Our theorem will require a mild assumption on the decision boundary $DB(f)$ of the classifier $f$, i.e., the set of points $x$ where $f(x)=0$.

\noindent\textbf{Assumption [A]:} For all $x \in DB(f)$ and all neighborhoods $U \ni x$ containing $x$, there exist points $x_{0}$ and $x_{1}$ in $U$ such that $f(x_{0}) < 0$ and $f(x_{1}) > 0$.

This assumption states that a point that is difficult to classify by $f$ has points of both labels in any given neighborhood around it. In particular, this means that the decision boundary does not contain an open set. We remark that both Assumption A and the continuity requirement for the score function $f$ are implicit in previous decomposition results like Equation 1 in \cite{zhang2019theoretically}. Without Assumption A, the ``neighborhood'' of the decision boundary in \cite{zhang2019theoretically} will not contain the decision boundary, and it is easy to give a counterexample to Equation 1 in \cite{zhang2019theoretically} if $f$ if not continuous.

Our decomposition result will decompose the adversarial risk into the normal and tangential directions: however, as we will show, an ``extra term'' appears, which we define next:

\begin{definition}[NNR Nearby-Normal-Risk]
Fix $\epsilon>0$. Denote by $A(x,y)$ the event that $\forall x' \in B^{nor}_{\epsilon}(x), f(x')y >0,$ i.e., the normal adversarial risk of $x$ is zero. 

Denote by $B(x,y)$ the event that \[\exists x' \in B^{in}_{2\epsilon}(x): (\exists z \in B^{nor}_{\epsilon}(x'): f(z)f(x') \leq 0),\] i.e., $x$ has a point $x'$ near it such that $x'$ has non-zero normal adversarial risk. 

Denote by $C(x,y)$ the event $\forall x' \in B^{in}_{2\epsilon}(x), f(x')y >0$, i.e., $x$ has no adversarial perturbation in the manifold within distance $2 \epsilon$.

The Nearby-Normal-Risk (denoted as NNR) of $f$ with budget $\epsilon$ is defined to be \[\displaystyle\mathop{\mathbb{E}}_{(x, y)\sim p} \mathds{1}(A(x,y) \wedge B(x,y) \wedge C(x,y)),\] where $\wedge$ denotes ``and''.
\end{definition}

We are now in a position to state our main result.

\begin{thm}\label{riskdecom}[Risk Decomposition]
Let $\mathcal{M}$ be a smooth compact manifold in $\mathbb{R}^{D}$, and let data be drawn from $\mathcal{M} \times \{-1,1\}$ according to some distribution $p$. There exists a $\Delta>0$ depending only on $\mathcal{M}$ such that the following statements hold for any $\epsilon< \Delta$. For any score function $f$ satisfying assumption A,

\begin{enumerate}[(i)]
    \item  \begin{eqnarray}\label{maineq1}
    R_{adv}(f,\epsilon) &\leq& R_{std}(f) + R^{nor}_{adv}(f,\epsilon) + R^{in}_{adv}(f,2\epsilon)\notag \\  
    &+&  \text{NNR}(f,\epsilon).
    \end{eqnarray}

    \item If $R_{adv}^{nor}(f,\epsilon)=0$, then \[R_{adv}(f,\epsilon) \leq R_{std}(f)+  R_{adv}^{in}(f,2\epsilon)\]
\end{enumerate}

\end{thm}

\noindent\textbf{Remark:} 

\begin{enumerate}
    \item The first result decomposes the adversarial risk into the standard risk, the normal adversarial risk, the in-manifold risk, and an ``extra term'' --- the Nearby-Normal-Risk. The NNR comes into play when a point $x$ doesn't have normal adversarial risk, and the score function on all points nearby agrees with $y(x)$, yet there is a point near $x$ that has non-zero normal adversarial risk.

\item The second result states that if the normal adversarial risk is zero, then the $\epsilon$-adversarial risk is bounded by the sum of the standard risk and the $2\epsilon$ in-manifold risk.

\end{enumerate}
One may wonder if a decomposition of the form  $R_{adv}(f,\epsilon) \leq  R_{std}(f)+ R^{nor}_{adv}(f,\epsilon)+ R^{in}_{adv}(f,2\epsilon)$ is possible. We prove that this is not possible.

\begin{thm}\label{tight}[Tightness of Decomposition Result]

    For any $\epsilon< 1/2$, there exists a sequence $\{f_{n}\}_{n=1}^{\infty}$ of continuous score functions such that 
    \begin{enumerate}
        \item $R_{std}(f) = 0$ for all $n \geq 1$,
        \item $R^{in}_{adv}(f_n,2\epsilon) = 0$ for all $n \geq 1$, and
        \item $R^{nor}_{adv}(f_n,\epsilon) \rightarrow 0$ as $n$ goes to infinity,
    \end{enumerate} 
    but $R_{adv}(f,\epsilon) = 1$ for all $n > \frac{1}{\sqrt{3}\epsilon}$.

\end{thm}

Thus all three terms except the NNR term go to zero, but the adversarial risk (the left side of Equation~\ref{maineq}) goes to one.

\subsection{Decomposition when y is Deterministic}

Let $\eta(x) = Pr(y=1|x)$. We consider here the simplistic setting when $\eta(x)$ is either zero or one, i.e., $y$ is a deterministic function of $x$. In this case, we can explain our decomposition result in a simpler way.

Let $Z^{nor}(f,\epsilon):=\{x \in \mathcal{M} : f(x)y >0 \text{ and } \exists x'\neq x \in B^{nor}_{\epsilon}(x), f(x')y(x) \leq 0\}$. That is, $Z^{nor}(f,\epsilon)$ is the set of points with no standard risk, but with a non-zero normal adversarial risk under a positive but less than $\epsilon$ normal perturbation. Let $\overline{Z^{nor}(f,\epsilon)} = \mathcal{M} \setminus Z^{nor}(f,\epsilon)$ be the complement of $Z^{nor}(f,\epsilon)$. For a set $A \subset \mathcal{M}$, let $\mu(A)$ denote the measure of $A$.

\begin{cor}\label{defriskdecom}
Let $\mathcal{M}$ be a smooth compact manifold in $\mathbb{R}^{D}$, and let $\eta(x) \in \{0,1\}$ for all $x \in M$. There exists a $\Delta>0$ depending only on $\mathcal{M}$ such that the following statements hold for any $\epsilon< \Delta$. For any score function $f$ satisfying assumption A,

\begin{enumerate}[(i)]
    \item  \begin{eqnarray} \label{maineq}
    R_{adv}(f,\epsilon) &\leq& R_{std}(f)+R^{in}_{adv}(f,2\epsilon) + R^{nor}_{adv}(f,\epsilon) \notag \\
    &+& \mu(\overline{Z^{nor}(f,\epsilon)} \cap B_{2\epsilon}(Z^{nor}(f,\epsilon))
    \end{eqnarray}
    
    \item If $R_{adv}^{nor}(f,\epsilon)=0$, then $R_{adv}(f,\epsilon) \leq R_{std}(f)+ R_{adv}^{in}(f,2\epsilon)$.
\end{enumerate}
\end{cor}

Therefore in this setting, the adversarial risk can be decomposed into the in-manifold risk and the measure of a neighborhood of the points that have non-zero normal adversarial risk.

\subsection{Proofs of Theorems~\ref{riskdecom} and \ref{tight}}

The complete proof of Theorem~\ref{riskdecom} is technical and is provided in the supplementary materials. Here we provide a sketch of the proof first. Then we give the complete proof of Theorem~\ref{tight}.

\subsubsection{Proof Sketch of Theorem~\ref{riskdecom}}

We first address the existence of the constant $\Delta$ that only depends on $\mathcal{M}$ in the theorem statement. Define a \textit{tubular neighborhood} of $\mathcal{M}$ as a set $\mathcal{N} \subset \mathbb{R}^D$ containing $\mathcal{M}$ such that any point $z \in \mathcal{N}$ has a unique projection $\pi(z)$ onto $\mathcal{M}$ such that $z-\pi(z) \in N_{\pi(z)} \mathcal{M}$. Thus the normal line segments of length $\epsilon$ at any two points $x,x' \in \mathcal{M}$ are disjoint. 

By Theorem 11.4 in \cite{bredon2013topology}, we know that there exists $\Delta$ such that $N:=\{y \in \mathbb{R}^{D}: dist(y,\mathcal{M}) < \Delta\}$ is a tubular neighborhood of $\mathcal{M}$. The $\Delta$ guaranteed by Theorem 11.4 is the $\Delta$ referred to in our theorem, and the budget $\epsilon$ is constrained to be at most $\Delta$.

For simplicity, we first sketch the proof of the case when $y$ is deterministic (the setting of Corollary 1).
 Consider a pair $(x,y) \sim p$, such that $x$ has an adversarial perturbation $x'$ within distance $\epsilon$. We show that one of the four cases must occur:
 
 \begin{itemize}
     \item $x'=x$ (standard risk).
     \item $x' \neq x$, $x' \in N_{x}\mathcal{M}$, and $f(x)y >0$ (normal adversarial risk).
     \item Let $x''=\pi(x')$ (the unique projection of $x'$ onto $\mathcal{M}$), then $d(x'',x) \leq 2 \epsilon$ and either
     \begin{itemize}
         \item $f(x'')y \leq 0$, and $x$ has an $2\epsilon$ in-manifold adversarial perturbation (in-manifold risk), or
         \item $f(x'')f(x') \leq 0$, which implies that $x$ is within $2\epsilon$ of a point $x'' \in \mathcal{M}$ that has non-zero normal adversarial risk. (NNR: nearby-normal-risk)
     \end{itemize}
 \end{itemize}

The second of these sets is $Z^{nor}(f,\epsilon)$ in the setting of Corollary~\ref{riskdecom}. One can see that the four cases correspond to the four terms in Equation~\ref{maineq}.

For the proof of Theorem~\ref{riskdecom}, one has to observe that since $y$ is not deterministic, the set $Z^{nor}(f,\epsilon)$ is random. One then has to average over all possible $Z^{nor}(f,\epsilon)$, and show that the average equals NNR. 

For the second part of Theorem 1 and Corollary 1, observe that if the normal adversarial risk is zero, then in the last case, $x''$ has non-zero normal adversarial risk, with normal adversarial perturbation $x'$. Unless $x''$ is on the decision boundary, by continuity of $f$ one can show that there exists an open set around $x''$ such that all points here have non-zero normal adversarial risk. This contradicts the fact that the normal adversarial risk is zero, implying that case 4 happens only on a set of measure zero (recall that by assumption A the decision boundary does not contain any open set).  This completes the proof sketch.

\subsubsection{Proof of Theorem~\ref{tight}}

Let $\mathcal{M} = [0,1]$ and fix $\epsilon <1/2$ and $n \geq 1$. We will think of data as lying in the manifold $\mathcal{M}$, and $\mathbb{R}^2$ as the ambient space. The true distribution is simply $\eta(x)=1$ for all $x \in \mathcal{M}$, hence $y \equiv 1$ (all labels on $\mathcal{M}$ are $1$).

Let $\ell_1=\frac{n-1}{n(n+1)}$ and $\ell_2 = \frac{1}{n^2}$. Note that $(n+1) \ell_1 + n \ell_2 = 1$. Consider the following partition of $\mathcal{M} = A_0 \cup B_1 \cup A_1 \cup B_2 \cup \cdots \cup B_n \cup A_n$, Where $A_{i}$ ($0 \leq i \leq n$) is of length $\ell_1$ and $B_i$ ($1 \leq i \leq n$) is an interval of length $\ell_2$. The interval $A_0, B_1, A_1,\cdots,B_n,A_n$ appear in this order from left to right.

For ease of presentation, we will consider $\{0,1\}$ binary labels, and build score functions $f_{n}$ taking values in $[0,1]$ that satisfy the conditions of the Theorem.

 For an $x \in A_{i}$ for some $0\leq i \leq n$, define  $g_{n}(x)=1$. For $x \in B_i$ for some $1 \leq i \leq n$, define $g_n(x)= \epsilon/2$. Observe that $\epsilon/2 < 1/4$. 

We now define the decision boundary of $f_{n}$ as the set of points in $\mathbb{R}^2$ on the ``graph'' of $g_n$ and $-g_n$. That is, \[ DB(f_{n}) = \left\{(x,cg_{n}(x)): x \in [0,1], c \in\{-1,1\}\right\}. \] 

See Figure~\ref{fig:lower_bd} for a picture of the upper decision boundary. Now let $f_{n}$ be any continuous function with decision boundary $DB(f_n)$ as above. That is, $f_{n}:\mathbb{R}^2 \rightarrow [0,1]$ is such that $f_{n}(x,t) > 1/2$ if $|t| < g_n(x)$, $f_{n}(x,t) < 1/2$ if $|t| > g_{n}(x)$ and $f_n(x,y) =1/2$ if $|t|=g_{n}(x)$.

\paragraph{In-manifold Risk Is Zero} Observe that since $\eta(x)=1$ on $[0,1]$, the in-manifold risk of $f_n$ is zero, since $f_n(x,0) > 1/2$, and so $\text{sign}(2f_n - 1)$ equals 1, which is the same as the label $y$ at $x$. This means that there are no in-manifold adversarial perturbations, no matter the budget. Thus $R^{in}_{adv}(f_n,\epsilon)=0$ for all $n \geq 1$.

\paragraph{Normal Adversarial Risk Goes To Zero} Next we consider the normal adversarial risk. If $x \in A_i$ for some $i$, then a point in the normal ball with budget $\epsilon$ is of the form $(x,t)$ with $|t| < \epsilon <1/2$, but $f_n(x,t)> 1/2$ for such points, and thus $\text{sign}(2f_n - 1) = y(x)$. Thus $x \in A_i$ does not contribute to the normal adversarial risk.

If $x \in B_i$ for some $i$ then $f_n(x, \epsilon) < 1/2$ while $f_n(x,0) >1/2$, and hence such $x$ contributes to the normal adversarial risk. Thus $R^{nor}_{adv}(f_n,\epsilon) = \sum_{i=1}^{n} \mu(B_i) = \sum_{i=1}^{n} \ell_{2} = 1/n$, which goes to zero as $n$ goes to infinity.

\paragraph{Adversarial Risk Goes To One} Now we show that $R_{adv}(f_n,\epsilon)$ goes to one. In fact, we will show that as long as $n$ is sufficiently large, the adversarial risk is 1. Consider $n$ such that $\ell_{1}:= \frac{n-1}{n(n+1)} < \sqrt{3}\epsilon$. Note that such an $n$ exists simply because $\ell_1$ goes to zero as $n$ goes to infinity, and $n > \frac{1}{\sqrt{3}\epsilon} $ works.

Clearly, points in $B_{i}$ contribute to adversarial risk as they have adversarial perturbations in the normal direction. However, if we consider $x \in A_{i}$ (which does not have adversarial perturbations in the normal direction or in-manifold), we show that there still exists an adversarial perturbation in the ambient space: that is, there exists a point $x'$ such that a) the distance between $(x',\epsilon/2)$ and $(x,0)$ is at most $\epsilon$, and b) $sign(2f_n(x,\epsilon/2)) \neq sign(2f_n(x,0))$. Let $x'$ be the closest point in $B:= \cup B_{i}$ to $x$. Then $|x'-x| \leq \ell_{1}/2 < \sqrt{3}\epsilon/2$. Thus the distance between $(x',\epsilon/2)$ and $(x,0)$ is at most $\sqrt{(\sqrt{3}\epsilon/2)^2+ (\epsilon/2)^2} = \epsilon$. Since $x' \in B$, $f_n(x',\epsilon/2) <1/2$ whereas $f_n(x,0) <1/2$, $(x',\epsilon/2)$ is a valid adversarial perturbation around $x$.

Thus for all $x \in [0,1]$, there exists an adversarial perturbation within budget $\epsilon$, and therefore $R_{adv}(f_n,\epsilon)=1$ as long as $n > \frac{1}{\sqrt{3}\epsilon}$. This completes the proof.

\begin{figure}
    \centering
    \includegraphics[width = 0.5\textwidth]{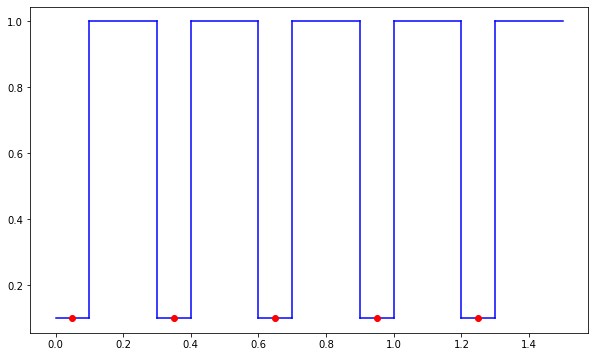}
    \caption{Lower bound illustration}
    \label{fig:lower_bd}
\end{figure}

\begin{table*}[h!]
\begin{center}
\caption{2D Adversarial risk comparison}
\label{table:2D_adv_table}
\begin{adjustbox}{width=0.9\columnwidth,center}
\begin{tabular}{|c|c|c|c|c|c|c|c|c|c|}
\hline
 Single Boundary & \multicolumn{2}{c|}{$f$} & \multicolumn{2}{c|}{$f^{adv}$} & \multicolumn{5}{c|}{$f^{nor}$} \\
\hline
$\epsilon$ & $R^{adv}$ & RHS &$R^{adv}$ & RHS &$R^{adv}$ & RHS & $R^{in}_{adv}(2\epsilon)$ & $R^{nor}_{adv}$ & $R_{std}$ \\
\hline
0.01 & 0.0110 & 0.022 & 0.0110 & 0.022 & 0.0090 & 0.0220 & 0.0140 & 0.0050 & 0.0050 \\
\hline
0.02 & 0.0130 & 0.0449 & 0.0130 & 0.0449 & 0.0130 & 0.0449 & 0.0290 & 0.0060 & 0.0060 \\
\hline
0.03 & 0.0230 & 0.063 & 0.0250 & 0.0671 & 0.0230 & 0.0633 & 0.0400 & 0.0120 & 0.0120 \\
\hline
0.05 & 0.0280 & 0.0794 & 0.0300 & 0.0784 & 0.0280 & 0.0794 & 0.0620 & 0.0040 & 0.0040 \\
\hline
0.1 & 0.0709 & 0.1652 & 0.0699 & 0.1645 & 0.0709 & 0.1650 & 0.133 & 0.0 & 0.0040 \\
\hline
0.15 & 0.0979 & 0.2831 & 0.1009 & 0.2886 & 0.1009 & 0.2866 & 0.1850 & 0.0050 & 0.0050 \\
\hline
0.2 & 0.128 & 0.3951& 0.126 & 0.3971& 0.128 & 0.4086 & 0.261 & 0.0050 & 0.0040 \\
\hline
0.25 & 0.1660 & 0.4966 & 0.1630 & 0.4931& 0.1660 & 0.4986 & 0.3259& 0.0040 & 0.0040 \\
\hline
0.3 & 0.1979 & 0.4509 & 0.1979 & 0.5613 & 0.1979 & 0.4505 & 0.35 & 0.0 & 0.0\\
\hline
\hline
Double Boundary & \multicolumn{2}{c|}{$f$} & \multicolumn{2}{c|}{$f^{adv}$} & \multicolumn{5}{c|}{$f^{nor}$} \\
\hline
$\epsilon$ & $R^{adv}$ & RHS &$R^{adv}$ & RHS &$R^{adv}$ & RHS & $R^{in}_{adv}(2\epsilon)$ & $R^{nor}_{adv}$ & $R_{std}$ \\
\hline
0.01 & 0.0080 & 0.0286 & 0.0060 & 0.0296 & 0.0070 & 0.0276 & 0.0180 & 0.0030 & 0.0030 \\
\hline
0.02 & 0.0240 & 0.0694 & 0.0230 & 0.2525 & 0.0240 & 0.0694 & 0.0490 & 0.0050 & 0.0050 \\
\hline
0.03 & 0.0510 & 0.1333 & 0.0460 & 0.1363 & 0.0510 & 0.1383 & 0.0949 & 0.0110 & 0.0110 \\
\hline
0.05 & 0.0620 & 0.1810 & 0.0620 & 0.1640 & 0.0629 & 0.1640 & 0.1139 & 0.0080 & 0.0080 \\
\hline
0.1 & 0.1170 & 0.3398 & 0.1169 & 0.3071 & 0.12 & 0.2746 & 0.2400 & 0.0060 & 0.0060 \\
\hline
0.15 & 0.1850 & 0.6059 & 0.1860 & 0.4895 & 0.1939 & 0.5948 & 0.3860 & 0.0040 & 0.0040 \\
\hline
0.2 & 0.242 & 0.8763 & 0.247 & 0.8002 & 0.265 & 0.878 & 0.5409 & 0.0060 & 0.0050 \\
\hline
0.25 & 0.3139 & 1. & 0.3169 & 0.9971 & 0.3239 & 1. & 0.6500 & 0.0080 & 0.0080 \\
\hline
0.3 & 0.386 & 0.9615 & 0.379 & 1. & 0.394 & 1. & 0.6520 & 0.0070 & 0.0060 \\
\hline
\end{tabular}
\end{adjustbox}
\end{center}
\end{table*}

\begin{table*}[!ht]
\begin{center}
\caption{3D Adversarial risk comparison}
\label{table:3D_adv_table}
\begin{adjustbox}{width=0.9\columnwidth,center}
\begin{tabular}{|c|c|c|c|c|c|c|c|c|c|}
\hline
 Single Boundary & \multicolumn{2}{c|}{$f$} & \multicolumn{2}{c|}{$f^{adv}$} & \multicolumn{5}{c|}{$f^{nor}$} \\
\hline
$\epsilon$ & $R^{adv}$ & RHS &$R^{adv}$ & RHS &$R^{adv}$ & RHS & $R^{in}_{adv}(2\epsilon)$ & $R^{nor}_{adv}$ & $R_{std}$ \\
\hline
0.1 & 0.0450 & 0.0992& 0.0410 & 0.092 & 0.0470 & 0.1002 & 0.0959 & 0.0050 & 0.0050 \\
\hline
0.2 & 0.1139 & 0.2297 & 0.0999 & 0.229 & 0.1099 & 0.2143 & 0.1929 & 0.0100 & 0.0199 \\
\hline
0.3 & 0.1550 & 0.3106 & 0.136 & 0.3216 & 0.1540 & 0.2852 & 0.239 & 0.0080 & 0.0265 \\
\hline
0.4 & 0.2089 & 0.3765 & 0.1680 & 0.3889 & 0.2059 & 0.3579 & 0.26 & 0.0080 & 0.0193 \\
\hline
0.5 & 0.247 & 0.4910 & 0.1860 & 0.4404 & 0.250 & 0.4104 & 0.252 & 0.0040 & 0.0174 \\
\hline
0.6 & 0.2700 & 0.5910 & 0.2179 & 0.5198 & 0.257 & 0.417 & 0.257 & 0.0090 & 0.0153\\
\hline
0.7 & 0.2600 & 0.6057 & 0.2009 & 0.7571 & 0.2731 & 0.4224 & 0.273 & 0.0030 & 0.0139 \\
\hline
0.8 & 0.2329 & 0.6775 & 0.1670 & 0.5630 & 0.2339 & 0.4083 & 0.2329 & 0.0020 & 0.0129\\
\hline
\hline
Double Boundary & \multicolumn{2}{c|}{$f$} & \multicolumn{2}{c|}{$f^{adv}$} & \multicolumn{5}{c|}{$f^{nor}$} \\
\hline
$\epsilon$ & $R^{adv}$ & RHS &$R^{adv}$ & RHS &$R^{adv}$ & RHS & $R^{in}_{adv}(2\epsilon)$ & $R^{nor}_{adv}$ & $R_{std}$ \\
\hline 
0.1 & 0.0649 & 0.1654 & 0.0789 & 0.153 & 0.0759 & 0.1654 & 0.1540 & 0.0130 & 0.0140 \\
\hline 
0.15 & 0.1460 & 0.3065 & 0.1280 & 0.2949 & 0.1510 & 0.3026 & 0.272 & 0.0220 & 0.0270 \\
\hline 
0.2 & 0.1700 & 0.3858 & 0.1370 & 0.3341 & 0.1670 & 0.3541 & 0.3040 & 0.0170 & 0.0170\\
\hline 
0.25 & 0.2049 & 0.4608 & 0.1500 & 0.4203 & 0.2099 & 0.4486 & 0.361 & 0.0210 & 0.0210 \\
\hline 
0.3 & 0.2159 & 0.4740 & 0.1810 & 0.4208 & 0.2119 & 0.4450 & 0.3289 & 0.0190 & 0.0190\\
\hline 
0.35 & 0.275 & 0.5176 & 0.2039 & 0.5289 & 0.2750 & 0.4830 & 0.356 & 0.0110 & 0.0130\\
\hline 
0.4 & 0.3000 & 0.6051 & 0.2069 & 0.5325 & 0.3040 & 0.6593 & 0.3690 & 0.0520 & 0.0080\\
\hline
\end{tabular}
\end{adjustbox}
\end{center}
\end{table*}

\begin{figure*}
  \centering
\begin{tabular}{p{3.5cm} p{3.5cm} p{3.5cm} p{3.5cm}}
\includegraphics[width=3cm]{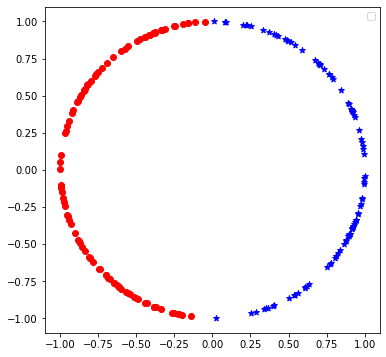}&
\includegraphics[width=3cm]{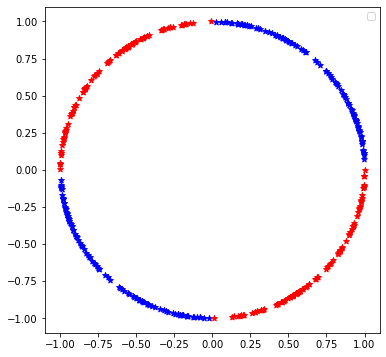}& 
\includegraphics[width=3cm]{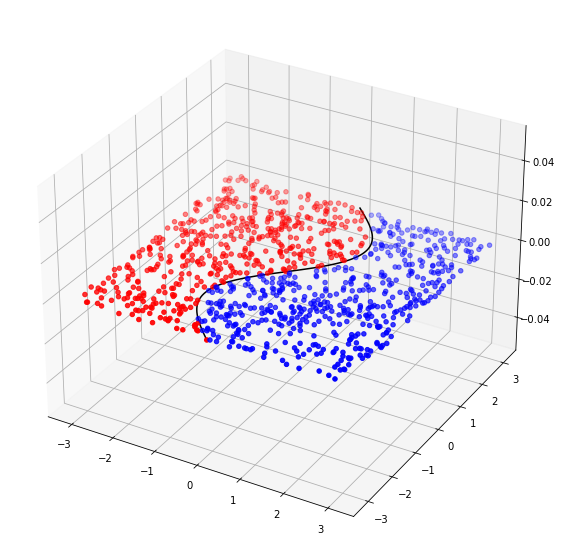}&
\includegraphics[width=3cm]{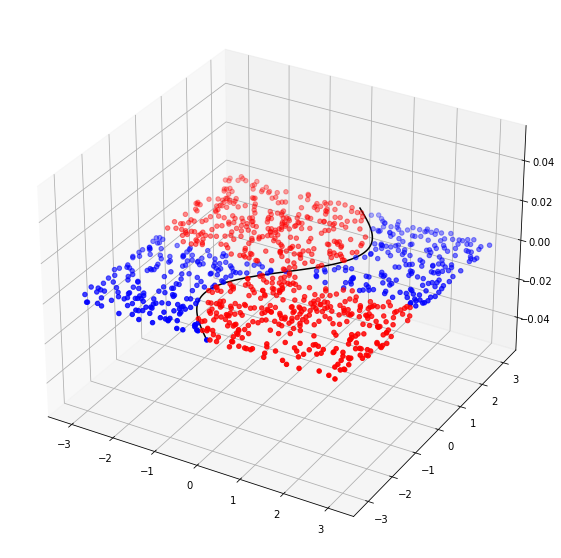}
\\
  \centering
\shortstack{a) 2D Single \\decision boundary} &
 \shortstack{b) 2D Double \\decision boundary} &
  \shortstack{c) 3D Single \\decision boundary} &
 \shortstack{d) 3D Double \\decision boundary}
\end{tabular}
\caption{In this figure, we show our four toy data set. On the left side is 2D data set on a unit circle. The single decision boundary data is linearly separated by the y-axis. And in the double decision boundary case, the circle is separated into 4 parts with x and y-axis. On the right side is the 3D data set. The data is distributed in a square area on $x_1x_2$-plane. In the single decision boundary example, the data is divided by the curve $x_1 = sin(x_2)$. And in the double decision boundary situation, we add the y-axis as the extra boundary. }
    \label{fig:2D_data}
\end{figure*}

\section{EXPERIMENT}
In this section, we verify the decomposition upper bound in Corollary~\ref{defriskdecom} on synthetic data sets. For both i) and ii) in Corollary~\ref{defriskdecom}, we empirically evaluate each term in the inequality on several classifiers and compare the values according to the claims in Corollary~\ref{defriskdecom}. 

In the following experiments, instead of using $l_2$ norm to evaluate the perturbation, we search the neighborhood under $l_{\infty}$ norm, which would produce a stronger attack than $l_2$ norm one. The experimental results indicate that our theoretical analysis may hold for an even stronger attack.


\subsection{Toy Data Set and Perturbed Data}
We generate four different data sets where 
we study both the \emph{single decision boundary case} and the \emph{double decision boundary case}. The first pair of datasets are in 2D space and the second pair is in 3D.  
We aim to provide empirical evidence for the claim $ii)$ in the Corollary~\ref{defriskdecom} using the single decision boundary data, having observed that one can sufficiently reduce $R^{nor}_{adv}$, allowing one directly compare $R_{adv}(f,\epsilon)$ and $R^{in}_{adv}(f,2\epsilon)$.
We aim to provide empirical evidence for the claim $i)$ in the Corollary~\ref{defriskdecom} using double boundary, since $R^{nor}_{adv}$ can not be sufficiently reduced using a simple classifier, since decision boundary is complicated.

For the 2D case, we sample training data uniformly from a unit circle $C_1: x_1^2 + x_2^2 =1$. For the single decision boundary data set, we set
\[
\begin{array}{rl}
     y = & 2\mathds{1}(x_1>0) -1  \text{ (Single Decision Boundary) }\\
     y = & 2\mathds{1}(x_1x_2 >0) -1  \text{ (Double Decision Boundary)}
\end{array}
\]
The visualization of the dataset is in Figure \ref{fig:2D_data} a) and b). In particular, we set unit circle $C_1$ has $\Delta = 1$, we set the perturbation budget to be $\varepsilon \in [0.01, 0.3]$. And the normal direction is alone the radius of the circle.

    

In the 3D case, we set the manifold to be  $\mathcal{M}: x_3 =0$ and generate training data in region $[-\pi,\pi ]\times [-\pi, \pi ]$ on $x_1x_2$-plane. We set
\[
\begin{array}{rl}
    y =  & 2\mathds{1}\left[x_1>sin(x_2)\right]-1 \text{(Single)} \\
    y = & 2 \mathds{1}\left[ (x_1-sin(x_2))x_2 >0\right]-1 \text{(Double)}
\end{array}
\]
Figure \ref{fig:2D_data} c) and d) show these two cases. For the single decision boundary example, due to the manifold being flat, we have $\Delta = \infty$,  we explore the $\epsilon$ value in range $[0.1, 0.8]$. For the double decision boundary, the distance to the decision boundary is half of the distance in the single boundary case. Therefore, we set the range of perturbation to be $[0.1, 0.4]$.

\subsection{Algorithm}

To empirically estimate the decomposition of adversarial risk, we need to generate adversarial data alone different directions, i.e. the normal direction risk $R^{nor}_{adv}$, the in-manifold risk $R^{in}_{adv}$ and the general adversarial risk $R_{adv}$. For general adversarial risk, we evaluate risks on perturbed example $x^{adv}$ computed by Projected Gradient Descent algorithm in~\citep{madry2017towards}.

By the definition of toy data sets, we know that the dimension of ambient space is 1. The normal space at point $x$ can be represented by $N_x{\mathcal{M}} = \{x + t\cdot v|0<t<\epsilon \}$, here $v$ is a unit normal vector. Therefore we could explicitly compute the normal vector $v$ and select normal direction adversarial data $x^{nor}$. 

We evaluate different components in the inequality in Corollary~\ref{defriskdecom} on three classifiers. The standard classifier $f$ trained by original training data set, the adversarial classifier $f^{adv}$ trained by Adversarial Training algorithm in~\cite{madry2017towards} and the classifier trained using adversarial samples generated in the normal direction $x^{nor}$, we denote it as $f^{nor}$. To compute the in-manifold perturbation, we design two methods. The first one is using grid search to go through all the perturbations in the manifold within the $\epsilon$ budget and return the point with maximum loss as in-manifold perturbation $x^{in}$. The second is using PGD method to find a general adversarial point $x^{adv}$ in ambient space and project $x^{adv}$ back to the data manifold $\mathcal{M}$. Due to grid-search being time-consuming, we use the second method in our experiments below. We further compare these two methods in supplementary materials. 

Here we propose a method based on Adversarial Training to compute $f^{nor}$ in Algorithm~\ref{alg:normal_AT}. One thing worth mentioning here is, instead of using grid search to find the actual $x^{nor}$, we use an intermediate method to generate normal data. We randomly choose a point along the normal direction within the $\epsilon$ budget to be our normal direction perturbed data. This might worsen the normal adversarial risk of $f^{nor}$, but our empirical results show that $R^{nor}_{adv}(f^{nor})$ still close to 0. 

\begin{algorithm}
	\caption{Normal direction Adversarial Training} 
	\label{alg:normal_AT}
	\begin{algorithmic}[1]
        \State Input: Training data set $\{x_i, y_i\}_{i=1}^n$, training iterations $K$, perturbation budget $\epsilon$
        \For {$iterations$ in $1,\ldots, K$}
            \For{$x_i$ in $\{x_i, y_i\}_{i=1}^n$}
                \State Find normal space $N_{x_i}(\mathcal{M})$
                for $x_i$.
                \State Find $v_i \in B^{nor}_{\varepsilon}(x_i)$ such that  $l(f(x_i +  v_i), y_i)\neq0$.
                \State $x^{nor}_i \leftarrow v_i$
                \State Update $f^{nor}$ with $x^{nor}_i $.
            \EndFor
        \EndFor
	\end{algorithmic} 
\end{algorithm}

Note in  Corollary~\ref{defriskdecom} claim $i)$, we have a component  $\mu(B_{\epsilon}(Z^{nor}(f, \epsilon))\cap \mathcal{M})$ on the right hand side (RHS) of the inequality. We also give a practical way of estimating such quantity in the empirical study. By the definition of $Z^{nor}(f, \epsilon)$, we first select point $x_i$ in training data such that there exists point $x_i^{\prime}$ in $N_{x_i}{\mathcal{M}}$ s.t $f(x_i^{\prime})\neq y_i$ to form set $\widehat{Z}$.  Since we uniformly sample points from data manifold, the volume of $B_{2\epsilon}(\widehat{Z})\cap \mathcal{M}$ is proportional to $\mu(B_{2\epsilon}(Z^{nor}(f, \epsilon))\cap \mathcal{M})$ which is a set derived by point wise augment $\widehat{Z}$ by a $2\varepsilon$-ball. In 2D example, this quantity is simply the length of curve segment on unit circle as the volume of $B_{2\epsilon}(z)$ for any $z\in\widehat{Z}$. In 3D example, we use area of $\widehat{Z}$ point wise augmented by a $2\epsilon$ square.
We list the RHS value for 2D and 3D datasets in Table~\ref{table:2D_adv_table} and Table~\ref{table:3D_adv_table} for all three classifiers. 

\subsection{Empirical Results and Discussion}

\textbf{2D Unit Circle} We generate 1000 training data uniformly. The classifier is a 2-layer feed-forward network. Each classifier is trained with Stochastic Gradient Descent (SGD) with a learning rate of $0.1$ for 1000 epochs. Also, since $\Delta=1$ for the unit circle, the upper bound of $\epsilon$ value is up to 1. Hence we run experiments for $\epsilon$ from 0.01 to 0.3. By increasing the $\epsilon$ budget, we also observe that the decision boundary of $f^{nor}$ becomes perpendicular to the data manifold. In Table~\ref{table:2D_adv_table}, the value of $R^{nor}_{adv}(f^{nor})$ also confirm our observation. We leave more discussion and visualization of this phenomenon in the supplementary material.

To verify our results in Corollary~\ref{defriskdecom} claim $i)$. We compute the adversarial risk for three classifiers. And for the upper bound, we evaluate the component $\mu(B_{\epsilon}(Z^{nor}(f,\epsilon))\cap\mathcal{M})$ following the description in Section 3.2. The right hand side value in the inequality is given in Table~\ref{table:2D_adv_table}. We could observe that the upper bounds hold for 2D data. 

Since we train $f^{nor}$ to minimize its empirical risk in normal direction. By Table~\ref{table:2D_adv_table}, we know $R^{nor}_{adv}(f^{nor})$ is close to zero. Therefore it is reasonable to study claim $ii)$ in Corollary~\ref{defriskdecom} using $f^{nor}$. The summation of in-manifold risk and standard risk of $f^{nor}$ certainly upper bounds $R_{adv}(f^{nor})$.

\textbf{3D $X_1X_2$-plane} We generate 1000 training data from the data set. The classifier is a 4-layer feedforward network. We use SGD with a learning rate of 0.1 and weight decay of 0.001 to train the network. The total training epoch is 2000. 

In Table~\ref{table:3D_adv_table}, we list same three classifiers trained on 3D data set. We have $R^{adv}$ been upper bounded by the right hand side of the inequality in Corollary~\ref{defriskdecom} claim $i)$. The claim $ii)$ also holds in 3D cases. 
Due to the limit of the space, we provide visualization of the decision boundary and additional empirical results in the supplemental material.

\section{CONCLUSION}
In this work, we study the adversarial risk of the machine learning model from the manifold perspective. We report theoretical results that decompose the adversarial risk into the normal adversarial risk, the in-manifold adversarial risk, and the standard risk with the additional Nearby Normal Risk term. We present a pessimistic case suggesting the additional Nearby Normal Risk term can not be removed in general, without additional assumptions. Observing that the Nearby Normal Risk term can be wiped out by enforcing zero normal adversarial risk,  our theoretical analysis suggests a potential training strategy that only focuses on the normal adversarial risk. 

\section*{Acknowledgements}
We thank anonymous reviewers for their constructive feedback. Mayank Goswami would like to acknowledge support from US National Science Foundation (NSF) awards CRII-1755791 and CCF-1910873.
Xiaoling Hu and Chao Chen were partially supported by grants NSF IIS-1909038 and CCF-1855760.

\bibliographystyle{plainnat}
\bibliography{example_paper.bib}

\clearpage

\begin{appendices}
\setcounter{thm}{0}
\setcounter{lemma}{0}
\setcounter{obs}{0}
\setcounter{prop}{0}
\section{PROOF OF THEOREM 1}

\begin{thm}[Risk Decomposition]
Let $\mathcal{M}$ be a smooth compact manifold in $\mathbb{R}^{D}$, and let data be drawn from $\mathcal{M} \times \{-1,1\}$ according to some distribution $p$. There exists a $\Delta>0$ depending only on $\mathcal{M}$ such that the following statements hold for any $\epsilon< \Delta$. For any score function $f$ satisfying assumption A,

\begin{enumerate}[(i)]
    \item $$
    R_{adv}(f,\epsilon) \leq R_{std}(f) + R^{nor}_{adv}(f,\epsilon) + R^{in}_{adv}(f,2\epsilon) +  \text{NNR}(f,\epsilon).
    $$

    \item If $R_{adv}^{nor}(f,\epsilon)=0$, then \[R_{adv}(f,\epsilon) \leq R_{std}(f)+  R_{adv}^{in}(f,2\epsilon)\]
\end{enumerate}

\end{thm}

\noindent\textbf{Proof of i):}
We first address the existence of the constant $\Delta$ that only depends on $\mathcal{M}$ in the theorem statement.

\begin{definition}[Tubular Neighborhood]
A \textit{tubular neighborhood} of a manifold $\mathcal{M}$ is a set $\mathcal{N} \subset \mathbb{R}^D$ containing $\mathcal{M}$ such that any point $z \in \mathcal{N}$ has a unique projection $\pi(z)$ onto $\mathcal{M}$ such that $z-\pi(z) \in N_{\pi(z)} \mathcal{M}$.
\end{definition}

By Theorem 11.4 in \cite{bredon2013topology}, we know that there exists $\Delta>0$ such that $N:=\{y \in \mathbb{R}^{D}: dist(y,\mathcal{M}) < \Delta\}$ is a tubular neighborhood of $\mathcal{M}$. This also implies that for any $0<\epsilon<\Delta$, the normal line segments of length $\epsilon$ at any two points $x,x' \in \mathcal{M}$ are disjoint, a fact that will be used later.

The $\Delta$ guaranteed by Theorem 11.4 is the $\Delta$ referred to in our theorem, and the budget $\epsilon>0$ is constrained to be at most $\Delta$.

Next we consider the left hand side, the adversarial risk:
\[
R_{adv}(f, \epsilon) := \displaystyle\mathop{\mathbb{E}}_{(x, y)\sim p}\mathds{1}(\exists x^{\prime}\in B_{\epsilon}(x): f(x')y \leq 0)
\]

Denote by $E(x,y)$ the event that $\exists x^{\prime}\in B_{\epsilon}(x): f(x')y \leq 0$.

We will write the indicator function above as the sum of indicator functions of four events. Specifically, define by $E_1(x,y),E_2(x,y),E_3(x,y),E_4(x,y)$ the following four events:

 \begin{itemize}
     \item $E_1(x,y)$: $f(x)y \leq 0$.
     
     \item $E_2(x,y)$: $f(x)y >0$ and $\exists x'\neq x \in B_{\epsilon}(x)$ such that $x'-x \in N_{x} \mathcal{M}$ and $f(x')y \leq 0$.
     \end{itemize}
     
     For the next two cases, let $x'\neq x \in B_{\epsilon}(x)$ be such that $x'-x \notin N_{x} \mathcal{M}$ and $f(x')y \leq 0$ (if such an $x'$ exists). Let $x''=\pi(x')$ be the unique projection of $x'$ onto $\mathcal{M}$. Note that $x'' \neq x$. Define:
     
     \begin{itemize}
         \item $E_3(x,y)$: $f(x'')y \leq 0$.
         \item $E_4(x,y)$: $f(x'')y > 0 \iff f(x'')f(x') \leq 0$.
     \end{itemize}
 
\begin{lemma}
\[\mathds{1}(E(x,y)) = \mathds{1}(E_1(x,y))+\mathds{1}(E_2(x,y))+\mathds{1}(E_3(x,y))+\mathds{1}(E_4(x,y)) \]

\end{lemma}

\begin{proof}
Assume $E(x,y)$ occurs, i.e, $\exists x^{\prime}\in B_{\epsilon}(x): f(x')y \leq 0$. Either $x'=x$ satisfies the condition (which is event $E_1$) or some $x' \neq x$ satisfies the condition. 

Now we further divide into the case when $f(x)y >0$ and $x'-x \in N_{x}\mathcal{M}$ (which is event $E_2$), or $f(x)y>0$ and $x'-x \notin N_{x}\mathcal{M}$. In the latter case, note that $x''=\pi(x')$ cannot equal $x$ as otherwise $x'-x$ would be in the normal space at $x$, since the projection map is unique inside the tubular neighborhood. Thus $x''$ is well-defined, and it is easy to see that the last two cases are disjoint and cover this remaining case. Thus we have shown that if $E(x,y)$ occurs, then one of the four disjoint events $E_i$ must occur, proving the lemma.

\end{proof}

Finally we have the following lemma, which completes the proof of the theorem after combining with Lemma~1.

\begin{lemma} The following relation holds between the risk and the expectation of the indicator functions in Lemma~1
\begin{enumerate}
    \item $\displaystyle\mathop{\mathbb{E}}_{(x, y)\sim p}\mathds{1}(E_1(x,y))=R_{std}(f)$
    \item  $\displaystyle\mathop{\mathbb{E}}_{(x, y)\sim p}\mathds{1}(E_2(x,y)) \leq R^{nor}_{adv}(f,\epsilon)$
    \item $\displaystyle\mathop{\mathbb{E}}_{(x, y)\sim p}\mathds{1}(E_3(x,y)) \leq R^{in}_{adv}(f,2\epsilon)$
    \item  $\displaystyle\mathop{\mathbb{E}}_{(x, y)\sim p}\mathds{1}(E_4(x,y)) \leq NNR(f,\epsilon)$
\end{enumerate}

\end{lemma}

\begin{proof}
1) and 2) follow by definitions of standard adversarial risk and normal adversarial risk, respectively. Consider the setting of $E_3(x,y)$: i.e., $f(x)y >0$, the adversarial perturbation $x'$ is not in the normal direction (so $f(x')y \leq 0$), and $f(x'')y \leq 0$. Observe that by the triangle inequality, $d(x,x'') \leq d(x,x')+d(x',x'') \leq \epsilon+\epsilon = 2\epsilon$, simply because a) $x'$ is within the $\epsilon$-ball of $x$, and b) $x''$ is closer to $x'$ than $x$.

This means that there is a point $x'' \in B^{in}_{2\epsilon} (x)$ such that $f(x'')y \leq 0$. The expectation over a random $(x,y)\sim p$ of this event is clearly at most $R^{in}_{adv}(f,2\epsilon)$ (the inequality need not be tight because $x$ may have adversarial perturbation within $2\epsilon$ and also satisfy some other events like $E_1$).

Lastly, by the definition of the NNR, we see that $A(x,y)$ occurs when $E_1(x,y)$ or $E_2(x,y)$ do not. Also $C(x,y)$ implies that the event $E_{3}(x,y)$ does not occur. We are now in the situation where $x''$ is within $2\epsilon$ of $x$, $f(x')y \leq0$, and $f(x'')y >0$. But this implies that $f(x'')f(x') \leq 0$, and since $x' \in B^{nor}_{\epsilon}(x'')$, it implies that $B(x,y)$ occurs. Thus all of $A(x,y)$, $B(x,y)$ and $C(x,y)$ occur, which is the definition of NNR.
\end{proof} 

\noindent\textbf{Proof of ii)}

If $R^{nor}_{adv}(f,\epsilon)=0$, we claim that $NNR(f,\epsilon)=0$. Setting these two terms to zero in i) proves ii).

Note that although $R^{nor}_{adv}(f,\epsilon)=0$, it does not imply that there are no normal adversarial perturbations for any $x$--- it just means that the measure of such $x$ with normal adversarial perturbation is zero. 

Also note that $R^{nor}_{adv}(f,\epsilon)=0$ does not exclude $A(x,y)$ or $C(x,y)$ from occurring (in fact $A$ occurs for almost all $x$). Thus the proof will focus on the measure of points where $B(x,y)$ can occur. We will prove the following lemma, which will complete the proof of the theorem.

\begin{lemma}
Let $(x,y)$ be such that $B(x,y)$ occurs, i.e., there exist $x' \in B_{\epsilon}(x)$ and $x''=\pi(x)$ such that $f(x')y \leq 0$, $f(x'')y >0$ and $d(x,x'') \leq 2 \epsilon$. Then $C(x,y)$ cannot occur, i.e., there exists a point $w \in B^{in}_{2\epsilon}(x)$ such that $f(w)y \leq 0$. Consequently, $NNR(f,\epsilon)=0$.
\end{lemma}

\begin{proof}
We first claim that if $B(x,y)$ occurs, it must be the case that $f(x'')=0$. Assuming this, if $f(x'')=0$, then by Assumption A we know there exists an $s \in B_{\epsilon}(x'') \cap B_{2 \epsilon}(x)$ such that $f(s)y \leq 0$, which imply that $C(x,y)$ cannot occur. This will complete the proof of the lemma.

To prove that $f(x'')=0$, consider what happens if $f(x'') \neq 0$. Assume first that $f(x') \neq 0$, and note that $f(x')f(x'') \leq 0$. By continuity of $f$, there exist open neighborhoods $U \ni x''$ and $V \ni x'$ such that $f$ has the same sign on all of $U$ and the same sign on all of $V$, i.e., $sign(f|U) =sign(f(x''))$ and $sign(f|V) = sign(f(x'))$.

Consider the normal bundle on $U$ defined as the set $U'=\{ y \in \mathcal{M}_{\Delta}: \pi(y) \in U\}$. In other words, $U'$ is the union of the normal line segments passing through points in $U$ (here $\mathcal{M}_{\Delta}$ denotes the tubular neighborhood of $\mathcal{M}$). Note that $U'$ is an open set.

Define $W'=U' \cap V$, and $W =\pi(W')$. $W \subset \mathcal{M}$ is an open set, but for every $w \in W$, there exists a point $w' \in W' \cap B^{nor}_{\epsilon}(w)$ such that $f(w')f(w) \leq 0$. Therefore there exists anormal adversarial perturbation for every point in $W$. Since the measure of $W$ is not zero, this contradicts the fact that $R^{nor}_{adv}(f,\epsilon)=0$.

The proof is completed by observing that in the remaining case when $f(x'') \neq 0$ but $f(x')=0$, there must exist (by assumption A) a point $w$ near $x'$ such that $f(w) \neq 0$ and $f(w)y <0$. This lands us in the previous case, which we showed contradicts the hypothesis that $R^{nor}_{adv}(f,\epsilon)=0$.
\end{proof}

\textbf{Remark}: In Corollary~\ref{defriskdecom}, $\mu(\overline{Z^{nor}(f,\epsilon)} \cap B_{2\epsilon}(Z^{nor}(f,\epsilon))$ is the NNR under deterministic case. Therefore, Corollary~\ref{defriskdecom} follows directly from the proof of Theorem 1.

\section{ADDITIONAL EXPERIMENTS}
In the main paper, we leave some experimental results to discuss in this supplementary materials. In the following section, we will first compare different ways of generating in-manifold attack data. In the later section, we compare the decision boundary of different classifiers. By visualization of the decision boundary, we aim to show that the defense training algorithm can defend the model against adversarial examples in the normal direction implying that the adversarial risk in the normal direction can be controlled. 

Also, we need to mark out that when we have a small $\epsilon$ value. The RHS for classifier $f^{nor}$ might be a little bit smaller than claim $ii)$ in Corollary~\ref{defriskdecom}. This is due to the fact that our way of computing measure $\mu(\overline{Z^{nor}(f,\epsilon)} \cap B_{2\epsilon}(Z^{nor}(f,\epsilon))$ includes the standard risk $R^{std}$ and normal risk $R^{nor}$. So the summation of $\mu(\overline{Z^{nor}(f,\epsilon)} \cap B_{2\epsilon}(Z^{nor}(f,\epsilon))$ and $R^{in}_{adv}$ forms the RHS. When we have a small $\epsilon$ value, the points with non-zero normal adversarial risk are concentrated near the decision boundary. Since we compute the $\mu(\overline{Z^{nor}(f,\epsilon)} \cap B_{2\epsilon}(Z^{nor}(f,\epsilon))$ based on the ratio between the length of line segment (or area of cube in 3D) and the circumference of unit circle (or area of $x_1x_2$-plane unit square), the ratio could be close to zero. Therefore, we might have the value of RHS smaller than the summation of $R^{in}_{adv} + R^{std} + R^{nor}_{adv}$. Aside from this, RHS still upper bounds $R^{adv}$ in all cases. 

\subsection{In-Manifold Attack Algorithm}
To estimate the \emph{in-manifold adversarial risk},  we have tested two potential algorithms for generating in-manifold adversarial examples. We present our observations on these two methods. Our empirical study in the paper leverage one of the two methods presented below, which generates a more powerful in-manifold adversarial example. 

One way to generate the adversarial samples is by brutal force. We use the grid search method to search the $B^{in}_{\epsilon}(x)$ region and find the maximum loss point in that region. We treat the maximum loss point as the in-manifold adversarial data. We call this approach the grid search method. Another approach we name as the projected method. We set the step size of the grid search proportional to the perturbation budget $\epsilon$. In general, we search 100 points in 1D cases and 400 points in the 2D manifold. In the projected method, we first use a general adversarial attack algorithm to generate adversarial data in ambient space. Then we project the generated adversarial example back to the manifold and return the results as our in-manifold adversarial data. In the following experiment, we use PGD as our generator of adversarial data in ambient space. Both methods will find in-manifold data that is adversarial to the given model. The rest of the experiment settings follow Section 3 in the main paper.
\setcounter{figure}{3}
\begin{figure*}[h!]
\begin{tabular}{p{8cm} p{8cm}}
\centering
\includegraphics[width=6.5cm]{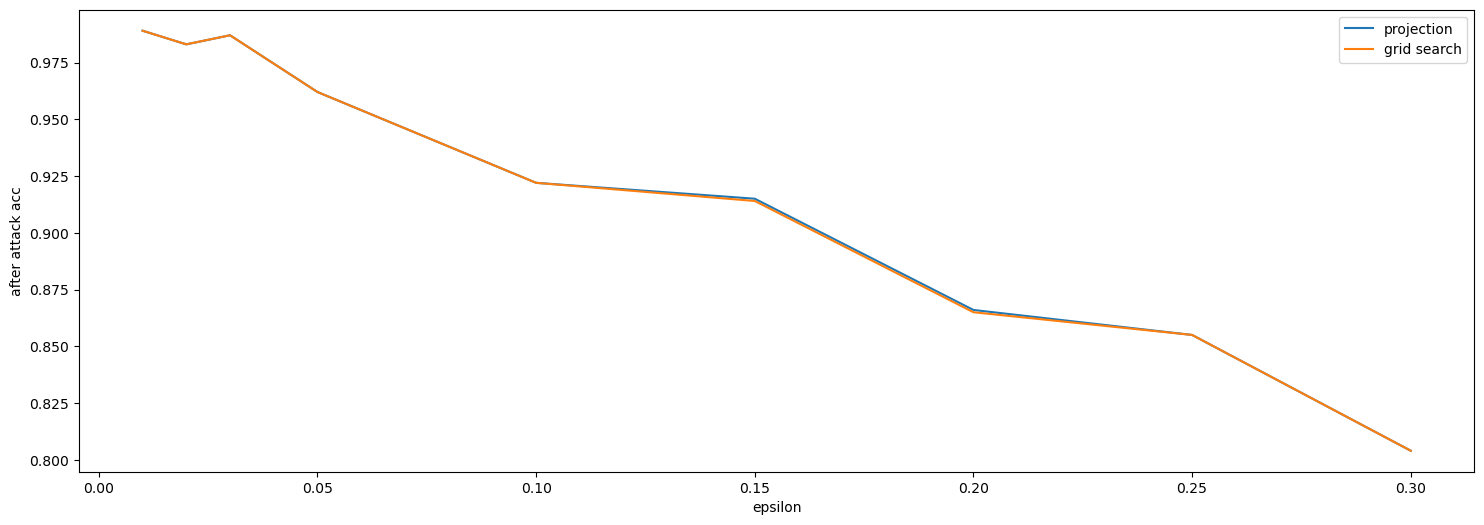}&
\includegraphics[width=6.5cm]{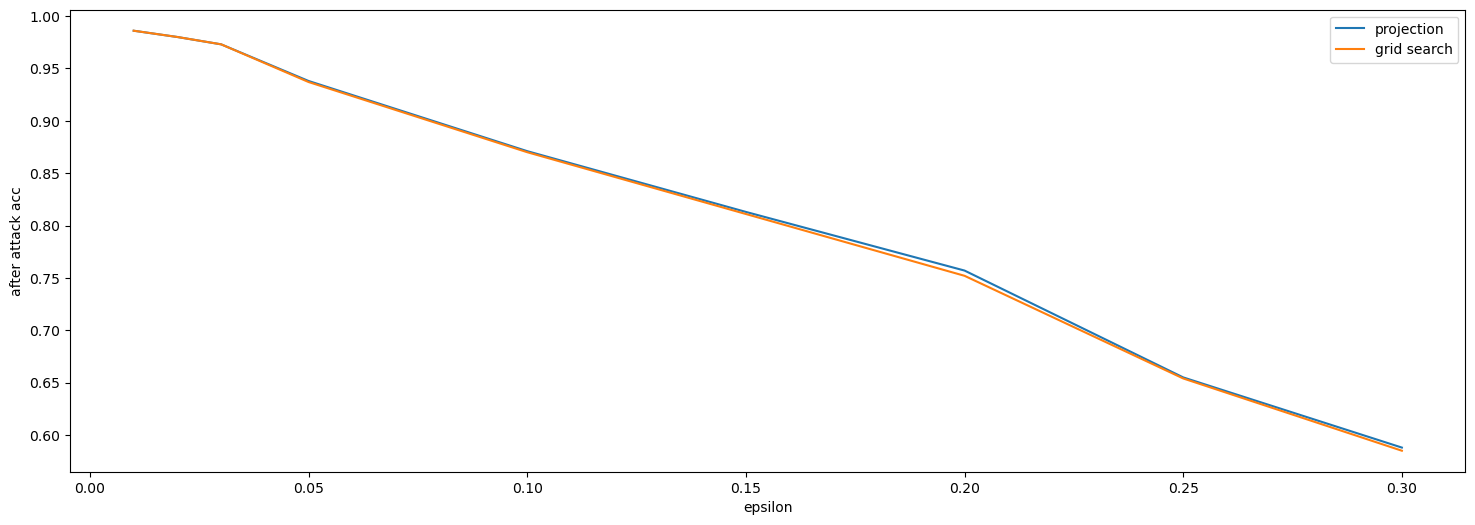}\\
\centering
\shortstack[c]{a) 2D Single \\decision boundary} &
$\;\;\;\;\;\;\;\;\;\;\;\;\;\;\;\;\;\;\;\;\;\;\;\;$\shortstack[c]{b) 2D Double \\decision boundary}\\
\centering
\includegraphics[width=6.5cm]{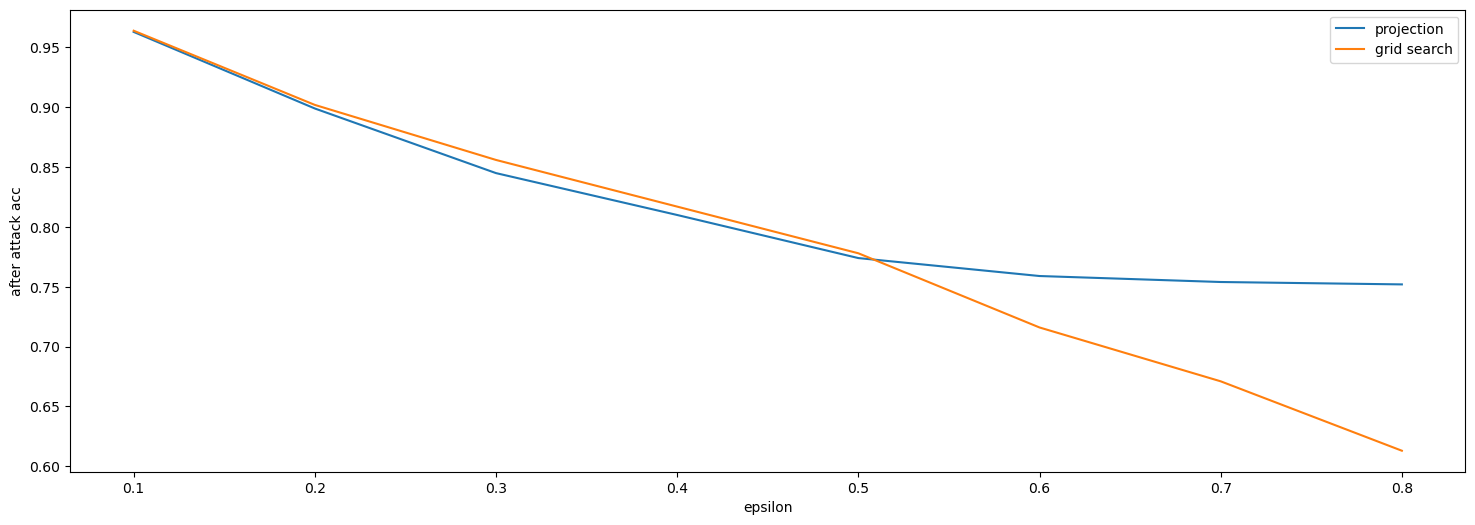} &
\includegraphics[width=6.5cm]{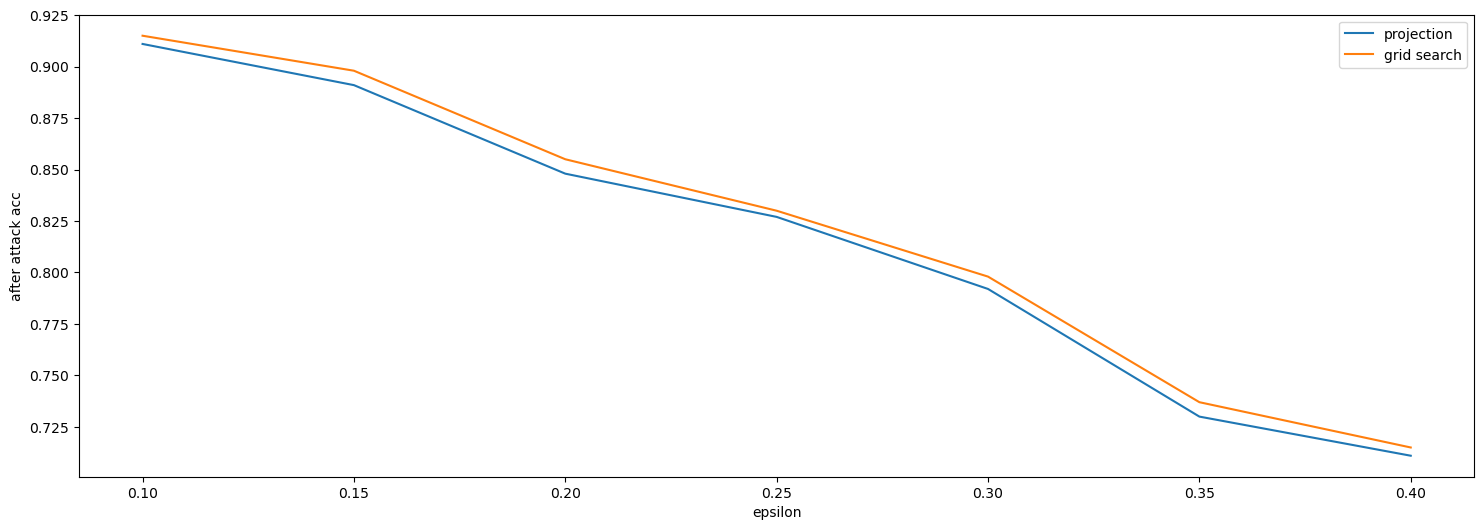}\\
\centering
\shortstack[c]{c) 3D Single \\decision boundary} &
$\;\;\;\;\;\;\;\;\;\;\;\;\;\;\;\;\;\;\;\;\;\;\;\;$\shortstack[c]{d) 3D Double \\decision boundary} 
\end{tabular}
\caption{We compare the grid search method and projection method to generate in-manifold attack data. The first row is after attack accuracy on the 2D data set. The blue line is the accuracy of the projection approach. Orange is for the grid search method. The $\epsilon$ range is smaller than the range we choose in the discussion of the main paper. This is because $\epsilon$-budget is larger than 0.05. The after-attack accuracy remains zero. The lower row is after attack accuracy on two different 3D data sets. }
    \label{fig:manifold_cmp}
\end{figure*}

In Figure~\ref{fig:manifold_cmp} we plot the after-attack accuracy of these two in-manifold attack methods. The experiments follow the same setting as the one we described in the main paper. We could observe that the grid search is slightly stronger in the 3D single boundary case and equivalent to the projection method in the rest of the cases. In the graph, the after-attack accuracy of the grid search method matches with the projection methods in the 2D case. And in the 3D case, when the $\epsilon$ is larger than $0.5$, then the grid search method achieves smaller after attack accuracy. This is due to the projection method searching the adversarial example in a smaller in-manifold ball. In other words, it hasn't fully explored the $\epsilon$ ball around the original data point. Therefore we could observe this small gap between these two methods. In the paper, we rely on the grid search method for generating in-manifold adversarial examples.

Furthermore, we compute the in-manifold risk in Table~\ref{table:2D_adv_table} and~\ref{table:3D_adv_table} using the grid search method. We plot our results in Table~\ref{table:2D_fnor_cmp} and Table~\ref{table:3D_fnor_cmp}. Since the attack performance of the grid search approach is stronger than the projection approach, the upper bound holds. In Table~\ref{table:2D_fnor_cmp} and Table~\ref{table:3D_fnor_cmp} we could observe this result. 

Comparing Table~\ref{table:2D_adv_table} and~\ref{table:2D_adv_table}, we could see that $R_{adv}^{in}$ in Table~\ref{table:2D_fnor_cmp} and Table~\ref{table:3D_fnor_cmp} has similar results. It implies that the projection method does not underestimate the upper in most cases. For the 3D double boundary dataset, the projection method has weaker results but the upper bound still holds. It implies that the upper bound in Corollary~\ref{defriskdecom} is loose in our study case. We could further prove a tighter upper bound in~\ref{riskdecom} claim $ii)$.
\setcounter{table}{2}
\begin{table*}[!ht]
\begin{center}
\caption{Computing $R^{in}_{adv}(f^{nor})$ using Grid Search methods for 2D data set}
\label{table:2D_fnor_cmp}
\begin{adjustbox}{width=0.9\columnwidth,center}
\begin{tabular}{|c|c|c|c|c|c|c|c|c|c|}
\hline
  Single Boundary & \multicolumn{2}{c|}{$f$} & \multicolumn{2}{c|}{$f^{adv}$} &\multicolumn{5}{c|}{$f^{nor}$} \\
\hline
$\epsilon$ & $R^{adv}$ & RHS & $R^{adv}$ & RHS & $R^{adv}$ & RHS & $R^{in}_{adv}(2\epsilon)$ & $R^{nor}_{adv}$ & $R_{std}$ \\
\hline
0.01 & 0.0110 & 0.0200 & 0.0110 & 0.022 & 0.0090 & 0.0180 & 0.0100 & 0.0050 & 0.0050 \\
\hline
0.02 & 0.0130 & 0.0426 & 0.0130 & 0.0425 & 0.0130 & 0.0439 & 0.0280 & 0.0060 & 0.0060 \\
\hline
0.03 & 0.0230 & 0.0499 & 0.0250 & 0.0595 & 0.0230 & 0.0613 & 0.0380 & 0.0120 & 0.0120 \\
\hline
0.05 & 0.0280 & 0.0871 & 0.0300 & 0.0881 & 0.0280 & 0.0843 & 0.0669 & 0.0040 & 0.0040 \\
\hline
0.1 & 0.0709 & 0.1974 & 0.0699 & 0.2026 & 0.0709 & 0.1620 & 0.1300 & 0.0 & 0.0040 \\
\hline
0.15 & 0.0979 & 0.2721 & 0.1009 & 0.3243 & 0.1009 & 0.3225 & 0.2209 & 0.0050 & 0.0050 \\
\hline
0.2 & 0.128 & 0.4063 & 0.126 & 0.4160 & 0.128 & 0.4206 & 0.2730 & 0.0050 & 0.0040 \\
\hline
0.25 & 0.1660 & 0.498 & 0.1630 & 0.5218 & 0.1660 & 0.5026 & 0.3299 & 0.0040 & 0.0040 \\
\hline
0.3 & 0.1979 & 0.6117 & 0.1979 & 0.6239 & 0.1979 & 0.5005 & 0.4000 & 0.0 & 0.0\\
\hline
\hline
 Double Boundary & \multicolumn{2}{c|}{$f$} & \multicolumn{2}{c|}{$f^{adv}$} &\multicolumn{5}{c|}{$f^{nor}$} \\
\hline
$\epsilon$ & $R^{adv}$ & RHS & $R^{adv}$ & RHS & $R^{adv}$ & RHS & $R^{in}_{adv}(2\epsilon)$ & $R^{nor}_{adv}$ & $R_{std}$ \\
\hline
0.01 & 0.0080 & 0.0404 & 0.0060 & 0.038 & 0.0070 & 0.0386 & 0.0290 & 0.0030 & 0.0030 \\
\hline
0.02 & 0.0240 & 0.0467 & 0.0230 & 0.0457 & 0.0240 & 0.0594 & 0.0390 & 0.0050 & 0.0050 \\
\hline
0.03 & 0.0510 & 0.1279 & 0.0460 & 0.1309 & 0.0510 & 0.1273 & 0.0839 & 0.0110 & 0.0110 \\
\hline
0.05 & 0.0620 & 0.1545 & 0.0620 & 0.1738 & 0.0629 & 0.1711 & 0.121 & 0.0080 & 0.0080 \\
\hline
0.1 & 0.1170 & 0.4037 & 0.1169 & 0.5155 & 0.12 & 0.3076 & 0.273 & 0.0060 & 0.0060 \\
\hline
0.15 & 0.1850 & 0.5649 & 0.1860 & 0.5619 & 0.1939 & 0.5768 & 0.368 & 0.0040 & 0.0040 \\
\hline
0.2 & 0.242 & 0.8709 & 0.247 & 0.82 & 0.265 & 0.88 & 0.5429 & 0.0060 & 0.0050 \\
\hline
0.25 & 0.3139 & 1. & 0.3169 & 1. & 0.3239 & 1. & 0.696 & 0.0080 & 0.0080 \\
\hline
0.3 & 0.386 & 1. & 0.379 & 1. & 0.394 & 1. & 0.833 & 0.0070 & 0.0060 \\
\hline
\end{tabular}
\end{adjustbox}
\end{center}
\end{table*}

\begin{table*}[!ht]
\begin{center}
\caption{Computing $R^{in}_{adv}(f^{nor})$ using Grid Search methods for 3D data set}
\label{table:3D_fnor_cmp}
\begin{adjustbox}{width=0.9\columnwidth,center}
\begin{tabular}{|c|c|c|c|c|c|c|c|c|c|}
\hline
 Single Boundary & \multicolumn{2}{c|}{$f$} & \multicolumn{2}{c|}{$f^{adv}$} &\multicolumn{5}{c|}{$f^{nor}$} \\
\hline
$\epsilon$ & $R^{adv}$ & RHS & $R^{adv}$ & RHS & $R^{adv}$ & RHS & $R^{in}_{adv}(2\epsilon)$ & $R^{nor}_{adv}$ & $R_{std}$ \\
\hline
0.1 & 0.0450 & 0.0974 & 0.0410 & 0.098 & 0.0470 & 0.0932 & 0.0889 & 0.0050 & 0.0050 \\
\hline
0.2 & 0.1139 & 0.2062 & 0.0999 & 0.2201 & 0.1099 & 0.2093 & 0.1879 & 0.0100 & 0.0199 \\
\hline
0.3 & 0.1550 & 0.3957 & 0.136 & 0.3557 & 0.1540 & 0.3482 & 0.3020  & 0.0080 & 0.0265 \\
\hline
0.4 & 0.2089 & 0.5124 & 0.1680 & 0.5008 & 0.2059 & 0.4729 & 0.375 & 0.0080 & 0.0193 \\
\hline
0.5 & 0.247 & 0.6057 & 0.1860 & 0.5405 & 0.250 & 0.6354 & 0.477 & 0.0040 & 0.0174 \\
\hline
0.6 & 0.2700 & 0.8444 & 0.2179 & 0.6828 & 0.257 & 0.7169 & 0.5569 & 0.0090 & 0.0153\\
\hline
0.7 & 0.2600 & 1. & 0.2009 & 0.8673 & 0.2731 & 0.8004 & 0.651 & 0.0030 & 0.0139 \\
\hline
0.8 & 0.2329 & 1. & 0.1670 & 1.  & 0.2339 & 0.8774 & 0.702 & 0.0020 & 0.0129\\
\hline
\hline
Double Boundary & \multicolumn{2}{c|}{$f$} & \multicolumn{2}{c|}{$f^{adv}$} & \multicolumn{5}{c|}{$f^{nor}$} \\
\hline
$\epsilon$ & $R^{adv}$ & RHS & $R^{adv}$ & RHS & $R^{adv}$ & RHS & $R^{in}_{adv}(2\epsilon)$ & $R^{nor}_{adv}$ & $R_{std}$ \\
\hline 
0.1 & 0.0649 & 0.1688 & 0.0789 & 0.1517 & 0.0759 & 0.1624 & 0.1510 & 0.0130 & 0.0140 \\
\hline 
0.15 & 0.1460 & 0.2581 & 0.1280 & 0.228 & 0.1510 & 0.2405 & 0.2099 & 0.0220 & 0.0270 \\
\hline 
0.2 & 0.1700 & 0.3476 & 0.1370 & 0.3174 & 0.1670 & 0.3441 & 0.2940 & 0.0170 & 0.0170\\
\hline 
0.25 & 0.2049 & 0.4700 & 0.1500 & 0.4300 & 0.2099 & 0.4576 & 0.37 & 0.0210 & 0.0210 \\
\hline 
0.3 & 0.2159 & 0.5745 & 0.1810 & 0.5240 & 0.2119 & 0.5331 & 0.4170 & 0.0190 & 0.0190\\
\hline 
0.35 & 0.275 & 0.5756 & 0.2039 & 0.5469 & 0.2750 & 0.555 & 0.4280 & 0.0110 & 0.0130\\
\hline 
0.4 & 0.3000 & 0.76 & 0.2069 & 0.7255 & 0.3040 & 0.8133 & 0.523 & 0.0520 & 0.0080\\
\hline
\end{tabular}
\end{adjustbox}
\end{center}
\end{table*}

\subsection{Decision Boundary Discussion}
In this section, we explain one of our intuitions of deriving this decomposing. In geometry, we know that if the decision boundary of the classifier is perpendicular to the manifold, then along normal direction, it is hard to find an adversarial example that can successfully attack the model. Therefore, the general adversarial risk is owing to tangential or in-manifold direction perturbation. Under this setting classifiers with decision boundary perpendicular to the manifold in ambient space would have $R^{nor}_{adv}$ equal zero. And this gives us claim $ii)$ in Theorem~\ref{riskdecom}. In the following section, we will plot the classifier's decision boundary in ambient space to state that our intuition holds on the synthetic data set.

\subsubsection{2D Decision Boundary}
In the 2D synthetic data set, we plot multiple decision boundaries of $f^{nor}$ in the double decision boundary case. As we increase the $\epsilon$ budget in the defense algorithm (Algorithm 1 in the paper), the decision boundary becomes more perpendicular to the unit circle. And it matches the results for $R^{nor}_{adv}(f^{nor})$ in Table 1. Around $\epsilon = 0.1$, $R^{nor}_{adv}(f^{nor})$ achieves the minimum value. And we could observe that the shape of the decision boundary is perpendicular and matches with the true label.
\setcounter{figure}{3} 
\begin{figure*}
  \centering
\begin{tabular}{p{5cm} p{5cm} p{5cm}}
\includegraphics[width=4.5cm]{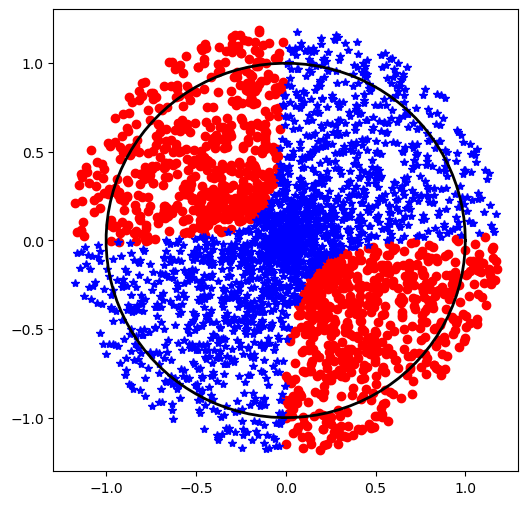}&
\includegraphics[width=4.5cm]{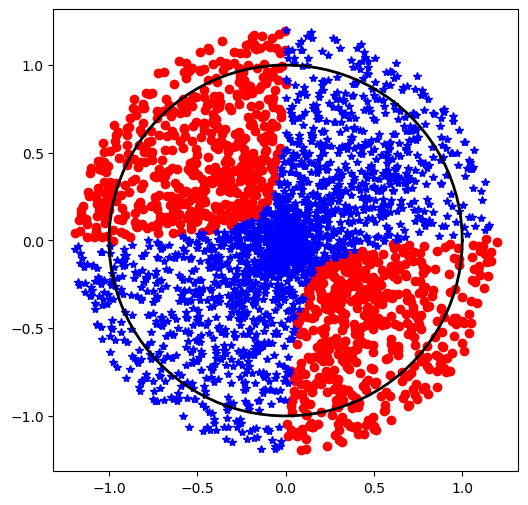}& 
\includegraphics[width=4.5cm]{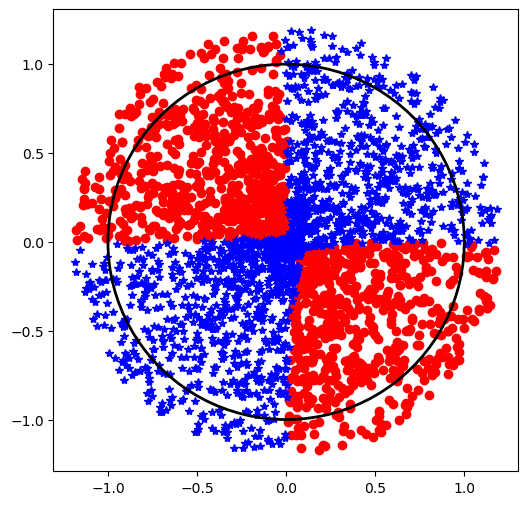}\\
  \centering
\shortstack{a)  $\epsilon = 0.02$ } &
 \centering \shortstack{b)  $\epsilon = 0.03$} &
  \centering \shortstack{c)  $\epsilon = 0.15$} 
\end{tabular}
\caption{In the graph, we first sample 3000 points in whole space and use $f^{nor}$ to classifier these sampled points. We use red dots and blue dots to mark two different classes. And the decision boundary of the classifier is easy to see in this setting. From left to right, we increase the $\epsilon$ budget from 0.02 to 0.15 and use the corresponding normal adversarial data to train the $f^{nor}$. The decision boundary of $f^{nor}$ is correlated with the size of $\epsilon$.}
    \label{fig:2D_bd}
\end{figure*}

\subsubsection{3D Decision Boundary}
In 3D cases, we plot the projection of points in ambient space back to the data manifold $x_1x_2$-plane. If the decision boundary is fully perpendicular to the $x_1x_2$-plane, the projection would have a clear separation and matches with the $x_2 = sin(x_1)$ boundary in the manifold. If not, we will have a region close to $x_2 = sin(x_1)$ with mixing red and blue points or the projection does not match with the in-manifold separation. 

We show the results in Figure~\ref{fig:3D_bd}. In the single decision boundary case, only $f^{nor}$ has (nearly) perpendicular decision boundary. 
For $f$, we can observe that the red points step into the region of the blue points and so does the blue points. And the adversarial training classifier $f^{adv}$ has an even worse result, its decision boundary does not fully match with the $x_2 = sin(x_1)$ curve inside the manifold, which implies that the classifier does not have good standard accuracy, which implies the trade-off between robustness and accuracy for the general robust classifier. And the same results and conclusions hold for the double boundary case. 
\begin{figure*}
  \centering
\begin{tabular}{p{5cm} p{5cm} p{5cm}}
\includegraphics[width=4.5cm]{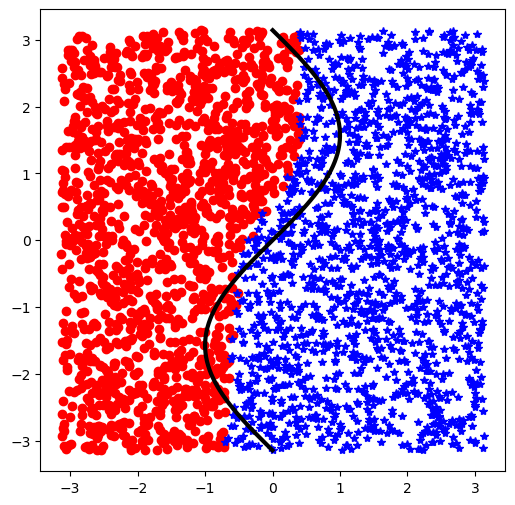}&
\includegraphics[width=4.5cm]{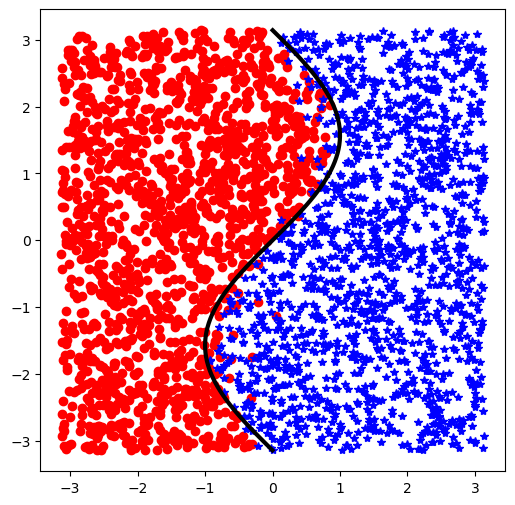}& 
\includegraphics[width=4.5cm]{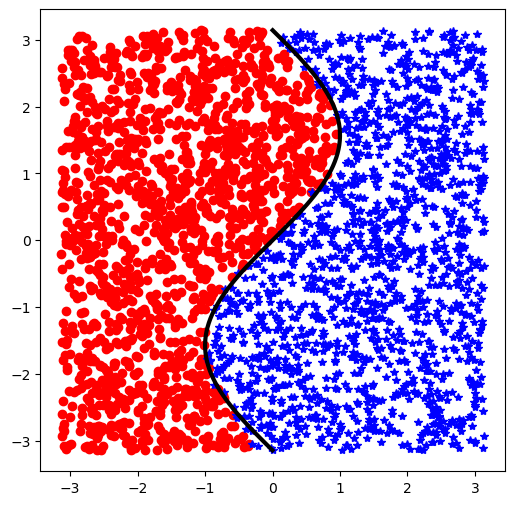}\\
\shortstack{a) Decision Boundary of $f^{adv}$ } &
 \shortstack{b)  Decision Boundary of $f$ } &
  \shortstack{c)  Decision Boundary of $f^{nor}$} \\
  \centering
  \includegraphics[width=4.5cm]{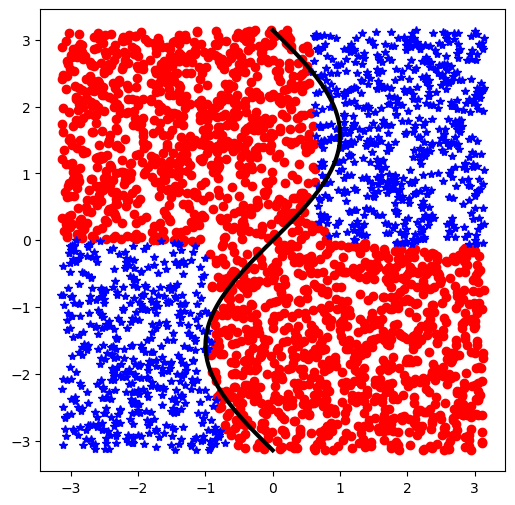}&
\includegraphics[width=4.5cm]{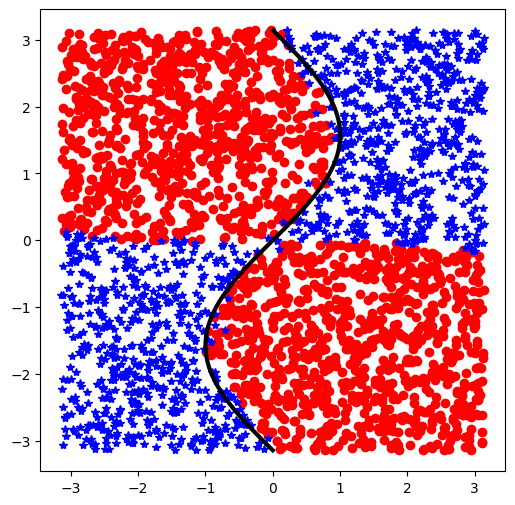}& 
\includegraphics[width=4.5cm]{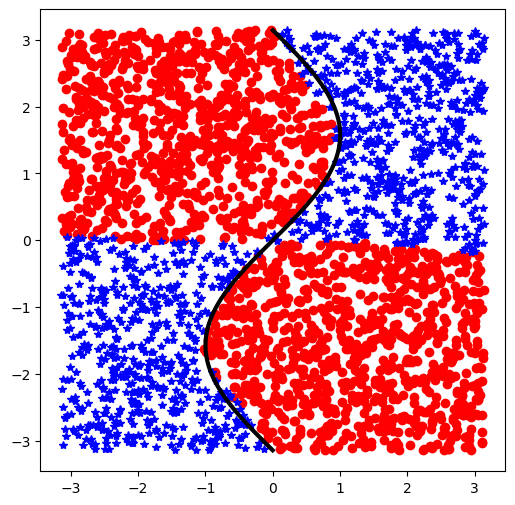} \\
\shortstack{a) Decision Boundary of $f^{adv}$} &
 \shortstack{b)  Decision Boundary of $f$  } &
  \shortstack{c)  Decision Boundary of $f^{nor}$ } \\
  
\end{tabular}
\caption{In this graph we show the projection of data classified with $f$, $f^{adv}$ and $f^{nor}$. We sample 5000 points in the tubular space of $x_1x_2$-plane. And use $f$, $f^{adv}$ and $f^{nor}$ to classify these points and mark with red and blue dots. If the point is in the ambient space, we project them back to the $x_1x_2$-plane. The first row are $f$, $f^{adv}$ and $f^{nor}$ trained with 3D single boundary synthetic data set. The second row is classifiers trained with the double boundary synthetic data set.}
    \label{fig:3D_bd}
\end{figure*}
\end{appendices}

\end{document}